\documentclass{article}

\usepackage{microtype}
\usepackage{graphicx}
\usepackage{subfigure}
\usepackage{booktabs} % for professional tables
\usepackage[table,xcdraw]{xcolor}

\usepackage{hyperref}

\usepackage[accepted]{icml2024}

\usepackage{amsmath}
\usepackage{amssymb}
\usepackage{mathtools}
\usepackage{amsthm}

\usepackage{enumitem}
\usepackage{wrapfig}
\usepackage{multirow}
\usepackage{listings}

\usepackage[capitalize,noabbrev]{cleveref}

\theoremstyle{plain}
\newtheorem{theorem}{Theorem}
\newtheorem{proposition}{Proposition}
\newtheorem{lemma}{Lemma}
\newtheorem{lemma*}{Lemma}

\theoremstyle{definition}

\theoremstyle{remark}

\usepackage[textsize=tiny]{todonotes}

\newcommand{\A}{\mathbf{A}}
\newcommand{\B}{\mathbf{B}}
\newcommand{\x}{\mathbf{x}}
\newcommand{\X}{\mathbf{X}}
\newcommand{\y}{\mathbf{y}}
\newcommand{\g}{\mathbf{g}}
\newcommand{\G}{\mathbf{G}}
\newcommand{\J}{\mathbf{J}}
\newcommand{\GG}{\mathcal{G}}
\newcommand{\ba}{\mathbf{a}}

\icmltitlerunning{A Linear Time and Space Local Point Cloud Geometry Encoder via Vectorized Kernel Mixture}

\begin{document}

\twocolumn[
\icmltitle{A Linear Time and Space Local Point Cloud Geometry Encoder\\via Vectorized Kernel Mixture (VecKM)}
% It is OKAY to include author information, even for blind
% submissions: the style file will automatically remove it for you
% unless you've provided the [accepted] option to the icml2024
% package.

% List of affiliations: The first argument should be a (short)
% identifier you will use later to specify author affiliations
% Academic affiliations should list Department, University, City, Region, Country
% Industry affiliations should list Company, City, Region, Country

% You can specify symbols, otherwise they are numbered in order.
% Ideally, you should not use this facility. Affiliations will be numbered
% in order of appearance and this is the preferred way.
% \icmlsetsymbol{equal}{*}

\begin{icmlauthorlist}
\icmlauthor{Dehao Yuan}{yyy}
\icmlauthor{Cornelia Fermüller}{yyy}
\icmlauthor{Tahseen Rabbani}{yyy}
\icmlauthor{Furong Huang}{yyy}
\icmlauthor{Yiannis Aloimonos}{yyy}
\end{icmlauthorlist}

\icmlaffiliation{yyy}{Department of Computer Science, University of Maryland, College Park, USA}
% \icmlaffiliation{comp}{Company Name, Location, Country}
% \icmlaffiliation{sch}{School of ZZZ, Institute of WWW, Location, Country}

\icmlcorrespondingauthor{Dehao Yuan}{dhyuan@umd.edu}
% \icmlcorrespondingauthor{Firstname2 Lastname2}{first2.last2@www.uk}

% You may provide any keywords that you
% find helpful for describing your paper; these are used to populate
% the "keywords" metadata in the PDF but will not be shown in the document
\icmlkeywords{Machine Learning, ICML}

\vskip 0.3in
]

% this must go after the closing bracket ] following \twocolumn[ ...

% This command actually creates the footnote in the first column
% listing the affiliations and the copyright notice.
% The command takes one argument, which is text to display at the start of the footnote.
% The \icmlEqualContribution command is standard text for equal contribution.
% Remove it (just {}) if you do not need this facility.

\printAffiliationsAndNotice{}  % leave blank if no need to mention equal contribution
% \printAffiliationsAndNotice{\icmlEqualContribution} % otherwise use the standard text.

\begin{abstract}
We propose VecKM, a local point cloud geometry encoder that is descriptive and efficient to compute. VecKM leverages a unique approach by vectorizing a kernel mixture to represent the local point cloud. 
Such representation's descriptiveness is supported by two theorems that validate its ability to reconstruct and preserve the similarity of the local shape. Unlike existing encoders downsampling the local point cloud, VecKM constructs the local geometry encoding using all neighboring points, producing a more descriptive encoding. Moreover, VecKM is efficient to compute and scalable to large point cloud inputs: VecKM reduces the memory cost from $(n^2+nKd)$ to $(nd+np)$; and reduces the major runtime cost from computing $nK$ MLPs to $n$ MLPs, where $n$ is the size of the point cloud, $K$ is the neighborhood size, $d$ is the encoding dimension, and $p$ is a marginal factor. The efficiency is due to VecKM's unique factorizable property that eliminates the need of explicitly grouping points into neighbors. In the normal estimation task, VecKM demonstrates not only 100x faster inference speed but also highest accuracy and strongest robustness. In classification and segmentation tasks, integrating VecKM as a preprocessing module achieves consistently better performance than the PointNet, PointNet++, and point transformer baselines, and runs consistently faster by up to 10 times.
\end{abstract}

\vspace{-12pt}
\section{Introduction}
\begin{figure}[t]
    \centering
    \includegraphics[width=0.49\textwidth]{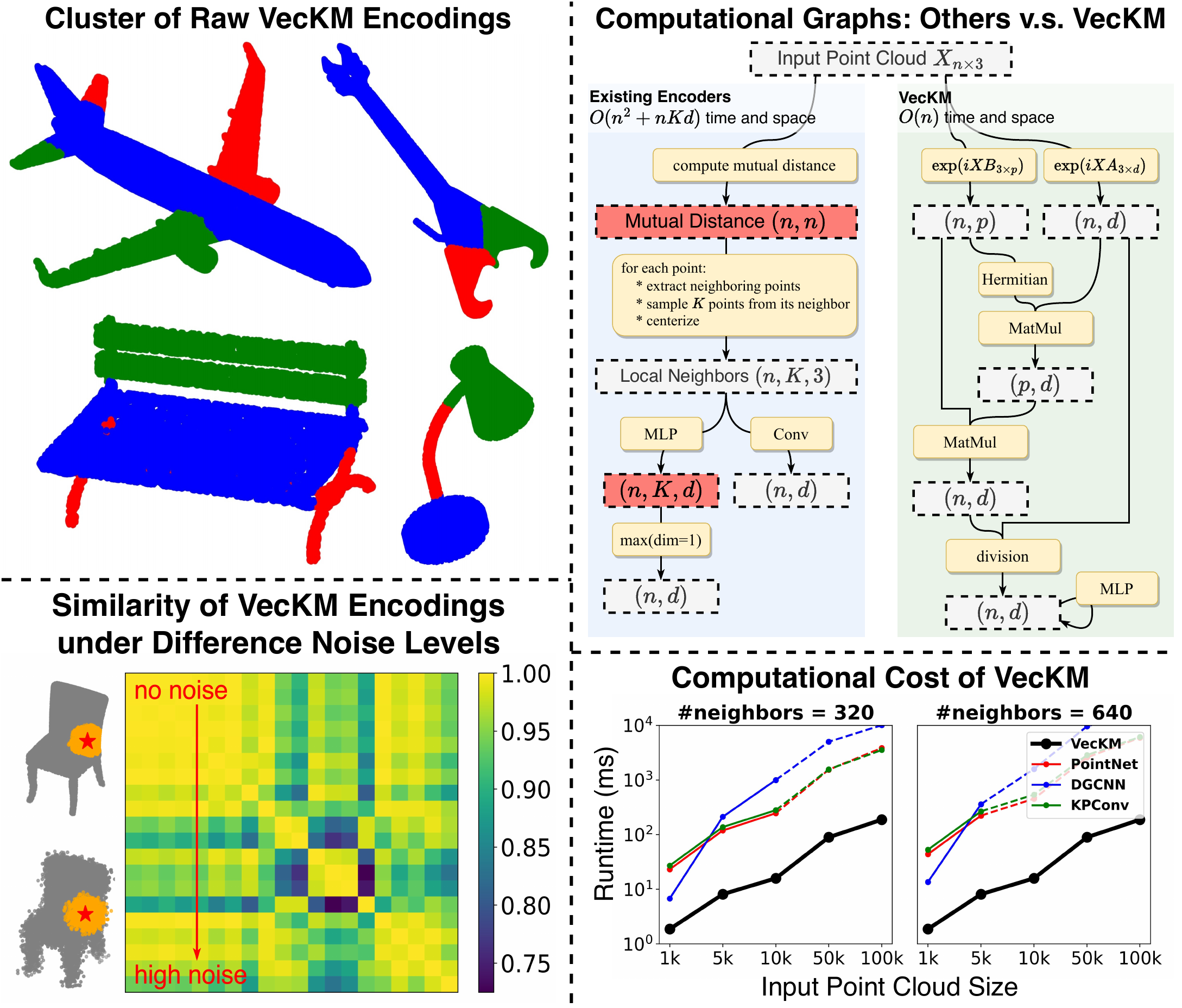}
    \vspace{-20pt}
    \caption{Our VecKM encoding is descriptive, robust to noise, and efficient in runtime and memory cost. \textbf{Upper Left}: Raw VecKM encodings, \emph{without any training}, already capture rich geometric features such as orientations and shapes. \textbf{Lower Left}: Under varying levels of noise, VecKM encodings remain highly consistent. \textbf{Upper Right}: Existing encoders face memory costs of $(n^2+nKd)$, while VecKM costs only $(nd+np)$ memory. Existing encoders compute $nK$ MLPs, whereas VecKM only computes $n$ MLPs. \textbf{Lower Right}: VecKM is 10x$\sim$100x faster than existing encoders in wall-clock time and scalable to large point cloud inputs.
    % This makes VecKM scalable, suitable for processing large point cloud input, and capable of handling large neighborhoods effectively.
    }
    \label{fig:teasor}
    \vspace{-10pt}
\end{figure}

The ubiquity and low cost of 3D sensors have drawn increased interest in the usage of three-dimensional point clouds for tasks such as autonomous driving \citep{sen2023scenecalib}, event cameras \citep{xiong2024event3dgs}, and remote sensing \citep{lu2020estimations, chen2022estimation}.

% Since raw point cloud data is an unordered collection of 3D points, 
In point cloud analysis, encoding local geometry is a fundamental step. In both low-level tasks such as feature matching and normal estimation, and high-level tasks such as classification, segmentation, and detection, encoding local geometry is usually required before passing the point cloud into any deep network. Much effort has been placed into the design of local geometry encoders, which can be loosely divided into two categories: hand-crafted features  and learnable encoders. Hand-crafted features \citep{han20233d} are manually defined features based on domain expertise, and learnable encoders require computationally expensive processing through trainable structures such as multi-layer perceptrons (MLP) \citep{qi2017pointnet, ma2022rethinking} or convolutions \citep{li2018pointcnn, thomas2019kpconv}.

These local geometry encoders follow a similar pipeline. They first group the input point cloud into neighborhoods and then process each neighborhood individually. As illustrated in Figure \ref{fig:teasor} (upper right), the pipeline involves computing the mutual distance between points. Then for each point, a number of $K$ points are sampled from its neighborhood, and MLP or convolution are used to transform the sampled neighborhood. In this pipeline, grouping the point cloud into neighborhoods requires $n^2$ time and space, and the MLP-based architectures, in particular, requires a sequence of MLPs to transform $nK$ vectors and reaches an intermediate stage of $(n, K, d)$. The pipeline results in bottlenecks in both computation and memory. Consequently, they usually resort to downsampling the local point clouds (i.e. reducing $K$), which can lead to inadequate representation of the local point cloud.

In this work, we address the computation and memory bottlenecks faced by the existing encoders, reducing the memory cost from $(n^2+nKd)$ to $(nd+np)$ and only computing $n$ MLPs. Besides, our representation is constructed from all the neighboring points without downsampling, and hence is more descriptive. Our approach is inspired by \citet{frady2022computing, yuan2023decodable}, which converts continuous functions into fixed-length vectors. Building on this concept, we introduce VecKM, which conceptualizes local point clouds as kernel mixtures (a form of continuous function) and vectorizes them. 
Under this formulation, we prove the local geometry encoding is reconstructive and isometric to the local point cloud, which guarantees the descriptiveness of the representation.
% This vectorized representation of kernel mixtures, being the local geometry encoding, is proved to be reconstructive and isometric to the original local point cloud. This theorem ensures the descriptiveness of the encoding. 
One essential advantage of VecKM is its factorizable property, which eliminates the need of explicitly grouping the neighborhoods and reuses many computations.

The VecKM encodings can subsequently be passed to deep cloud point models, such as PointNet++ \citep{qi2017pointnet++} and transformers \citep{vaswani2017attention}. VecKM's light representation and ease of computation significantly speed up the inference, while still achieving on-par or improved performance than other networks in classification and segmentation tasks. Our contributions are summarized below: 
\begin{itemize}[leftmargin=7pt,topsep=0pt,itemsep=-0.5mm]
    \item We present VecKM, a local geometry encoder that is descriptive and efficient. VecKM costs only $nd+np$ memory and computes only $n$ MLPs. This is achieved through a novel approach of vectorizing kernel mixtures, coupled with its unique factorizability. VecKM is the only existing local geometry encoder that costs linear time and space.
    \item Unlike existing encoders downsampling the local point cloud, 
    VecKM constructs the local geometry encoding using all neighboring points, and hence is more descriptive.
    \item  We evaluate our VecKM on multiple point cloud tasks. In normal estimation, VecKM is $>100$x faster and achieves $>16\%$ lower error than other widely-used learnable encoders and demonstrates the strongest robustness against different types of data corruption. 
    In classification and segmentation tasks, integrating VecKM as a preprocessing module achieves consistently better performance than the PointNet, PointNet++, and point transformer baselines, and runs consistently faster by up to 10 times.
    % In classification and part segmentation, integrating VecKM with PointNet, PointNet++ and transformers significantly speeds up inference time by up to $9.5$x, while still achieving consistently better performance. In semantic segmentation, VecKM improves the PointNet++ baseline by 3.4\% mIoU and improves the point transformer baseline by 20\% inference speed.
\end{itemize}

\section{Related Work}
\subsection{Local Geometry Encoder}
\label{sec:related_work_local}
The initial processing of raw point cloud data, an unordered collection of points, typically begins with the extraction of local geometric features. This step is essential before any further processing can occur. Existing methods for encoding local geometry can generally be categorized into two groups: hand-crafted features and learning-based encoders. All the encoders require grouping point clouds into neighborhoods.
% Since a raw point cloud is simply an unordered collection of points in $\mathbb{R}^3$, it must be encoded into feature vectors before being passed through a learnable architecture. 

\textbf{Hand-Crafted Features } are manually defined features that describe the local geometry. Domain expertise of the point cloud dataset and the task of interest are usually needed for constructing those features. We refer the reader to \citet{han20233d} for a comprehensive survey of hand-crafted features.

Histogram-based features represent a significant category within hand-crafted features, which transform the local point cloud into specific coordinate systems such as Cartesian \cite{prakhya20173dhopd}, polar \cite{ge2016non}, star-shaped \cite{steder2010narf}. Then the coordinate system is quantized, and the local point cloud is binned accordingly. The resulting feature is formed by concatenating the histograms. While these features induce minimal information loss, the resolution of the point cloud affects the quality of such features.

Statistics-based features form another category within hand-crafted features, which construct statistical descriptors from geometric parameters. Examples include eigenvalues \cite{vandapel2004natural}, covariance \cite{fehr2012compact}, normal orientation distribution \cite{triebel2006robust}, angles between normals \cite{rusu2008persistent}, local umbrella shapes \cite{ran2022surface}. These types of features can include rich geometric features and easily achieve rotation invariance. But they tend to be lossy and are often not robust to noise and variations in density.

\textbf{Learning-Based Encoders} transform the raw local point clouds into fixed-length vectors through trainable networks. These encoders can be broadly categorized into MLP-based encoders and convolution-based encoders.

MLP-based encoders use MLPs to transform the local point cloud and use max-pooling to cast the point cloud into a fixed-length vector. These operations can be performed repeatedly to retrieve deep point cloud features \cite{qi2017pointnet++}. Examples of MLP-based encoders include PointNet \cite{qi2017pointnet}, CurveNet \cite{xiang2021walk}, PointMLP \cite{ma2022rethinking}. MLP-based encoders are usually faster to compute than convolution-based ones. But they require computing an intermediate step of $(n, K, d)$ to perform the max-pooling operation. So they induce high memory cost when the input size $n$ and the neighborhood size $K$ is large.

Convolution-based encoders use point or edge convolution to transform the local point cloud into a fixed-length vector. Examples include KPConv \cite{thomas2019kpconv}, PointCNN \cite{li2018pointcnn}, PointConv \cite{wu2019pointconv}, SpiderCNN \cite{xu2018spidercnn}. Convolution-based encoders are more expensive to compute, but they do not encounter the memory bottleneck faced by MLP-based encoders.

\vspace{-1pt}
VecKM is both a hand-crafted feature and a learning-based encoder. It not only captures the geometric features, but also faithfully encodes the point distribution. This duality allows VecKM to leverage the strengths of both approaches. 

\vspace{-2pt}
\subsection{Point Cloud Architectures}
\vspace{-1pt}
We describe two major families of architectures for processing point clouds: PointNet++ and transformers. VecKM, as to be shown later, is compatible with both architectures.

\vspace{-1pt}
\textbf{PointNet++} \cite{qi2017pointnet++} utilizes hierarchical neural layers to capture fine geometric details at multiple scales. Within each layer, PointNets are utilized to transform the features. Many works have been done to improve the architecture. Examples include using different learning-based local geometry encoders, as introduced in Section \ref{sec:related_work_local}, and improving the neighborhood grouping strategies \cite{xiang2021walk, yan2020pointasnl}. PointNet++ and its derivatives tend to be faster than but less accurate than transformers.

\vspace{-1pt}
\textbf{Transformers.}
Given their success on a wide variety of vision tasks, along with their tolerance to permutations, many transformer-based models have been proposed for 3D point cloud processing. Models such as PCT \citep{guo2021pct}, 3CROSSNet \cite{han20223crossnet}, and Point-BERT \cite{yu2022point} apply transformer blocks to individual points to extract global information. Other models such as Point Transformer (PT) \cite{zhao2021point}, Pointformer \citep{pan20213d}, and the Stratified Transformer \cite{lai2022stratified} process local patches to extract local feature information. However, transformer-based models can suffer from computational and memory bottlenecks \cite{han20233d} as the attention map increases in size.

\vspace{-3pt}
\section{Methodology}
\vspace{-1pt}
\subsection{Problem Definition and Main Theorems}
\vspace{-1pt}
\textbf{Problem Definition.} Let the input point cloud be $X=\{\x_k\}_{k=1}^n$. Denote the centerized neighbor of the point $\x_k$ as $\mathfrak{N}(\x_k):=\{\x_j-\x_k: ||\x_j-\x_k||<r\}$. The output is the set of dense local geometric features $G=\{\mathbf{g}_k\}_{k=1}^n$, where $\mathbf{g}_k=E\big(\mathfrak{N}(\x_k)\big)\in\mathbb{C}^d$. We look for an encoder $E$ that maps the local point cloud into a fixed-length vector, which ``captures the underlying shape" sampled by the point cloud.

\vspace{-1pt}
To better formalize the heuristic expression of ``capturing the underlying shape", we think of the local shape around the point $x_k$ as a distribution function $f_k:\mathbb{R}^3\rightarrow\mathbb{R}^+$, where $f_k(\x)$ gives the probability density that a point $\x$ is on the local shape. We then think of the centerized local point cloud $\mathfrak{N}(\x_k)$ as random samples from the distribution function $f_k$. We expect the local point cloud encoding $E\big(\mathfrak{N}(\x_k)\big)\in\mathbb{C}^d$ to represent the distribution function $f_k$. For a good representation, we consider two natural properties: 1. the distribution function can be reconstructed from the encoding; 2. the correlation of the distribution functions is preserved by the similarity of the encodings.

\textbf{Pointwise Local Geometry Encoding.} Under the problem definition, we present the formula for encoding the local geometry around a single point. Unless specified otherwise, all input points $\x_j$ are assumed to be three-dimensional.

\begin{theorem}[Pointwise Local Geometry Encoding]
    Denote the neighbors of the point $\x_0$ as $\mathfrak{N}(\x_0):=\{\x_k-\x_0\}_{k=1}^n$. The local geometry encoding of $\x_0$ is computed as
    \vspace{-5pt}
    \begin{equation}
        E_\A\big(\mathfrak{N}(\x_0)\big)=\frac{1}{n}\sum_{k=1}^n \exp\big(i (\x_k-\x_0)\A_{3\times d}\big)
        \label{eqn:LGE}
        \vspace{-18pt}
    \end{equation}
    \label{thm:LGE}
\end{theorem}
\textit{where $i$ is the imaginary unit and $\mathbf{A}\in\mathbb{R}^{3\times d}$ is a fixed random matrix where each element follows the normal distribution $\mathcal{N}(0,\alpha^2)$.} As to be shown in Section \ref{sec:pointwise_local_geometry}, $E_\A\big(\mathfrak{N}(\x_0)\big)$ is fundamentally vectorizing a kernel mixture about $\mathfrak{N}(\x_0)$, so we name the encoding VecKM. Next, we present two propositions that claim VecKM encoding produces a good representation of the local shape:

\begin{proposition}[Reconstruction]
    WLOG, let $f$ be the distribution function characterizing the local shape of $\mathbf{0}$, $X=\{\x_k\}_{k=1}^n$ be the random samples drawn from the distribution function $f$, and $\g_n=\frac{1}{n}\sum_{k=1}^n \exp\big(i \x_k\A\big)$ be the VecKM encoding given by Eqn. \eqref{eqn:LGE}. $\A\in\mathbb{R}^{3\times d}$ is a fixed matrix whose entries are drawn from $\mathcal{N}(0, \alpha^2)$. Then at all points $\x$ where $f(\x)$ is continuous, as $n\to\infty$ and $\alpha^2\to0$,
    \vspace{-3pt}
    \begin{equation*}
        \langle\g_n, \exp(i\x\A)\rangle \rightarrow f(\x)
    \end{equation*}
    \label{thm:reconstruction}
    \vspace{-24pt}
\end{proposition}
\textit{where $\langle \cdot, \cdot \rangle$ denotes the inner product between two complex vectors.} 
% \tah{I understand what you are trying to say here but the notation is bit confusing. You may want to write $\langle\g, \exp(i\Theta x)\rangle$ as a function $g_n(x)$ or something in this vein.} 
The proposition states that under a suitable selection of the parameter $\alpha^2$, the distribution function $f$ can be approximately reconstructed from the VecKM encoding $\g_n$.

\begin{proposition}[Similarity Preservation]
    Let $f_1$, $f_2$ be two distribution functions characterizing two local shapes and $X_1$, $X_2$ be the random samples from the two distribution functions. $\g_1$, $\g_2$ are the VecKM encodings given by Eqn. \eqref{eqn:LGE} with $X_1$, $X_2$ as inputs. $\A$ is a fixed matrix whose entries are drawn from $\mathcal{N}(0, \alpha^2)$. Then the function similarity is preserved by the VecKM encodings: as $n\to\infty$ and $\alpha^2\to 0$,
    \vspace{-13pt}
    \begin{align*}
        \langle \g_1, \g_2 \rangle \rightarrow \langle f_1, f_2 \rangle = \int_{\mathbb{R}^3} f(x)g(x)dx
    \end{align*}
    \vspace{-16pt}
    \label{thm:similarity_preserving}
\end{proposition}
\textit{where $\langle \cdot, \cdot \rangle$ denotes the inner product between two complex vectors.}
% \tah{what are the functions $f$ and $g$ in this context?} 
The proposition states that under a suitable selection of the parameter $\alpha^2$, the correlation of functions (i.e. shapes) is approximately preserved by the VecKM encoding.

In brief, Theorem \ref{thm:LGE} presents the formula for encoding the local geometry around a single point. Proposition \ref{thm:reconstruction}, \ref{thm:similarity_preserving} assert that VecKM well represents the underlying local geometry. In Section \ref{sec:pointwise_local_geometry}, we will explain the mechanism behind Theorem \ref{thm:LGE} and prove Proposition \ref{thm:reconstruction}, \ref{thm:similarity_preserving} in detail.

\textbf{Dense Local Geometry Encoder} With Eqn. \eqref{eqn:LGE}, we can already compute the local geometry encoding for each point individually by grouping their neighborhoods. However, VecKM has a unique factorizable property that enables us to reuse computations and eliminate the intermediate step:

\begin{theorem}[Dense Local Geometry Encoding]
    Denoting the input point cloud as a matrix $\X_{n\times 3}=\left[\x_1; \x_2; \cdots; \x_n\right]$, the dense local geometry encoding $\G_{n\times d}$ is computed by
    \vspace{-0pt}
    \begin{align}
    \begin{split}
        \mathcal{A}_{n\times d}&=\exp(i\X_{n\times 3}\A_{3\times d}) \\
        \mathcal{B}_{n\times p}&=\exp(i\X_{n\times 3}\B_{3\times p}) \\
        \G_{n\times d}&=normalize\big((\mathcal{B}\times\mathcal{B}^H\times\mathcal{A}) \: ./ \: \mathcal{A}\big)
    \end{split}
    \label{eqn:DLGE}
    \end{align}
    \label{thm:DLGE}
\end{theorem}
\vspace{-12pt}

\textit{where $\A$ and $\B$ are two random fixed matrix whose entries are drawn from $\mathcal{N}(0,\alpha^2)$ and $\mathcal{N}(0,\beta^2)$. $\times$ denotes the matrix multiplication, and $./$ denotes the elementwise division.} As to be explained in Section \ref{sec:dense_local_geometry}, computing the dense local geometry encoding using Eqn. \eqref{eqn:DLGE} has almost the same effect as computing the pointwise local geometry encoding using Eqn. \eqref{eqn:LGE}. However, Eqn. \eqref{eqn:DLGE} only takes $\Theta(npd)$ time and $(np+nd)$ space to compute, where $p$, to be shown, is a marginal factor. The computation graph is visualized in Figure \ref{fig:teasor} (upper right).

\textbf{Structure of Proof.} In Section \ref{sec:pointwise_local_geometry}, we explain the mechanism behind Theorem \ref{thm:LGE} and prove our assertion that VecKM produces a good representation of the local geometry. In Section \ref{sec:dense_local_geometry}, we explain why Eqn. \eqref{eqn:DLGE} has almost the same effect as Eqn. \eqref{eqn:LGE} and the mechanism behind Theorem \ref{thm:DLGE}. In Section \ref{sec:VecKM_deep}, we introduce how to incorporate VecKM encodings into deep point cloud architectures.

\subsection{Pointwise Local Geometry Encoder}
\label{sec:pointwise_local_geometry}
In this section, we introduce why VecKM (Eqn. \ref{eqn:LGE}) produces a good representation of the local geometry. The key idea, as illustrated in Figure \ref{fig:proof_outline}, is that (i) VecKM vectorizes a Gaussian kernel mixture associated with the local point cloud, where (ii) the associated kernel mixture can approximate the local shape distribution function. Therefore, VecKM effectively represents the local shape. We will separately validate assertion (i) and (ii).

\textbf{(i) VecKM vectorizes a kernel mixture.} We first present a lemma stating that VecKM embodies a Gaussian kernel $\GG$:

\begin{lemma}[VecKM embodies a Gaussian kernel]
Let $\x, \y\in\mathbb{R}^3$, $\A\in\mathbb{R}^{3\times d}$. All elements in $\A$ are drawn from normal distribution $\mathcal{N}(0, \alpha^2)$. Then as $d\to\infty$,
\vspace{-3pt}
\begin{align*}
\begin{split}
    \frac{1}{d}\langle e^{i\x\A}, e^{i\y\A}\rangle\rightarrow \GG_\alpha(\x, \y):=\exp\big(-\frac{\alpha^2||\x-\y||^2}{2}\big)
\end{split}
\end{align*}
\label{thm:lemma1}
\end{lemma}

\vspace{-14pt}
Lemma \ref{thm:lemma1} is a corollary from the Bochner's theorem \citep{bochner2005harmonic, rahimi2007random}. We provide a detailed proof in Appendix \ref{sec:app_lemma1}. Importantly, the Gaussian kernel $\GG$ is approximated by the inner product of finite-length vectors $e^{i\x\A}$ and $e^{i\y\A}$. This approximation is important in vectorizing the kernel mixture and ensures the reconstructive and isometric properties in Proposition \ref{thm:reconstruction}, \ref{thm:similarity_preserving}, as detailed in the subsequent two lemmas. Unless otherwise specified, all entries in $\A$ are drawn from $\mathcal{N}(0, \alpha^2)$ and $\GG$ means $\GG_\alpha$.  The proofs are borrowed from \citet{frady2022computing}.

\begin{figure}
    \centering
    \includegraphics[width=0.4\textwidth]{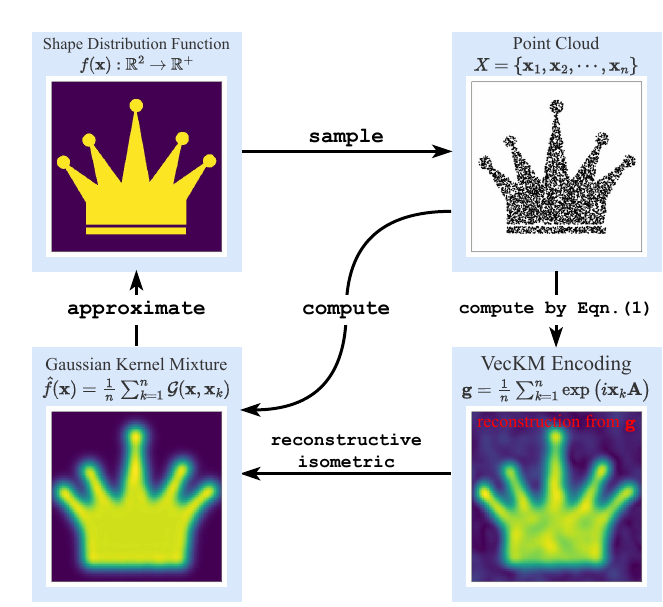}
    \vspace{-3pt}
    \caption{Theoretical outline of VecKM illustrated by 2d shapes. A point cloud, sampled from a shape distribution function, is associated with a Gaussian kernel mixture and a corresponding VecKM encoding, where the VecKM encoding is proved to be reconstructive and isometric to the Gaussian kernel mixture. Since the Gaussian kernel mixture can approximate the shape function, the VecKM encoding yields a good representation of the shape.}
    \label{fig:proof_outline}
    \vspace{-4pt}
\end{figure}

\begin{lemma}[Reconstruction]
    Let $\g=\frac{1}{n}\sum_{k=1}^n\exp(i \x_k \A)$ be the VecKM encoding, where all entries in $\A\in\mathbb{R}^{3\times d}$ are drawn from $\mathcal{N}(0, \alpha^2)$. $\hat{f}(\x)=\frac{1}{n}\sum_{k=1}^n \GG_\alpha(\x,\x_k)$ be the associated Gaussian kernel mixture. Then $\langle \exp(i \x\A),\g\rangle\rightarrow \hat{f}(\x)$ as $d\to\infty$. 
    % \tah{In equation 3, $x$ is a scalar, while $x_k$ is a 3d vector}
    \label{thm:reconstruction_lemma}
\end{lemma}
The lemma is derived from the linearity of the inner product:
\begin{align*}
    \langle \exp(i\x\A), \g\rangle&=\frac{1}{n}\sum_{k=1}^n \langle \exp(i\x\A), \exp(i\x_k\A) \rangle\\
    &\rightarrow\frac{1}{n}\sum_{k=1}^n \GG(\x, \x_k)=\hat{f}(\x)
\end{align*}
The lemma states that the Gaussian kernel mixture can be approximately reconstructed from the VecKM encoding $\g$, which theoretically shows that VecKM is equivalent to the Gaussian kernel mixture when $d$ is large.

\begin{lemma}[Similarity Preservation]
    Let $\g_1$, $\g_2$ be two VecKM encodings and $f_1$, $f_2$ be their associated Gaussian kernel mixtures. Then $\langle \g_1, \g_2 \rangle\rightarrow\langle f_1, f_2 \rangle$ as $d\to\infty$.
    \label{thm:similarity_preservation_lemma}
\end{lemma}
The lemma states that the VecKM encoding preserves the similarity/correlation between kernel mixtures, which further verifies that the encoding is not only equivalent but also isometric to the Gaussian kernel mixture. The lemma is derived from the linearity of integration:
% \vspace{-3pt}
\begin{align*}
    \langle f_1, f_2\rangle &= \int_{\x\in\mathbb{R}^3} \Big(\frac{1}{n}\sum_{p=1}^n \GG(\x,\x_p)\Big) \Big(\frac{1}{m}\sum_{q=1}^m \GG (\x,\x_q')\Big)d\x\\
    &=\frac{1}{mn}\sum_{p, q} \int_{\x\in\mathbb{R}^3} \GG(\x,\x_p)\GG(\x,\x_q') d\x\\
    &=\frac{1}{mn}\sum_{p, q} \GG(\x_p,\x_q') \leftarrow \langle \g_1, \g_2 \rangle
\end{align*}

% \tah{the kernel functions here take one argument, but in all the other Lemmas they take in two arguments}
\vspace{-3pt}
Lemma \ref{thm:lemma1}-\ref{thm:similarity_preservation_lemma} complete the argument that the VecKM encoding is equivalent and isometric to the kernel mixture when $d$ is large. In practice, the selection of $d$ is independent of the size of the point cloud. $d$ as small as 256 yields good encoding in many scenarios, for example, in our experiments.

\begin{wrapfigure}{r}{0.23\textwidth}
  \centering
  \vspace{-12pt}
  \includegraphics[width=0.23\textwidth]{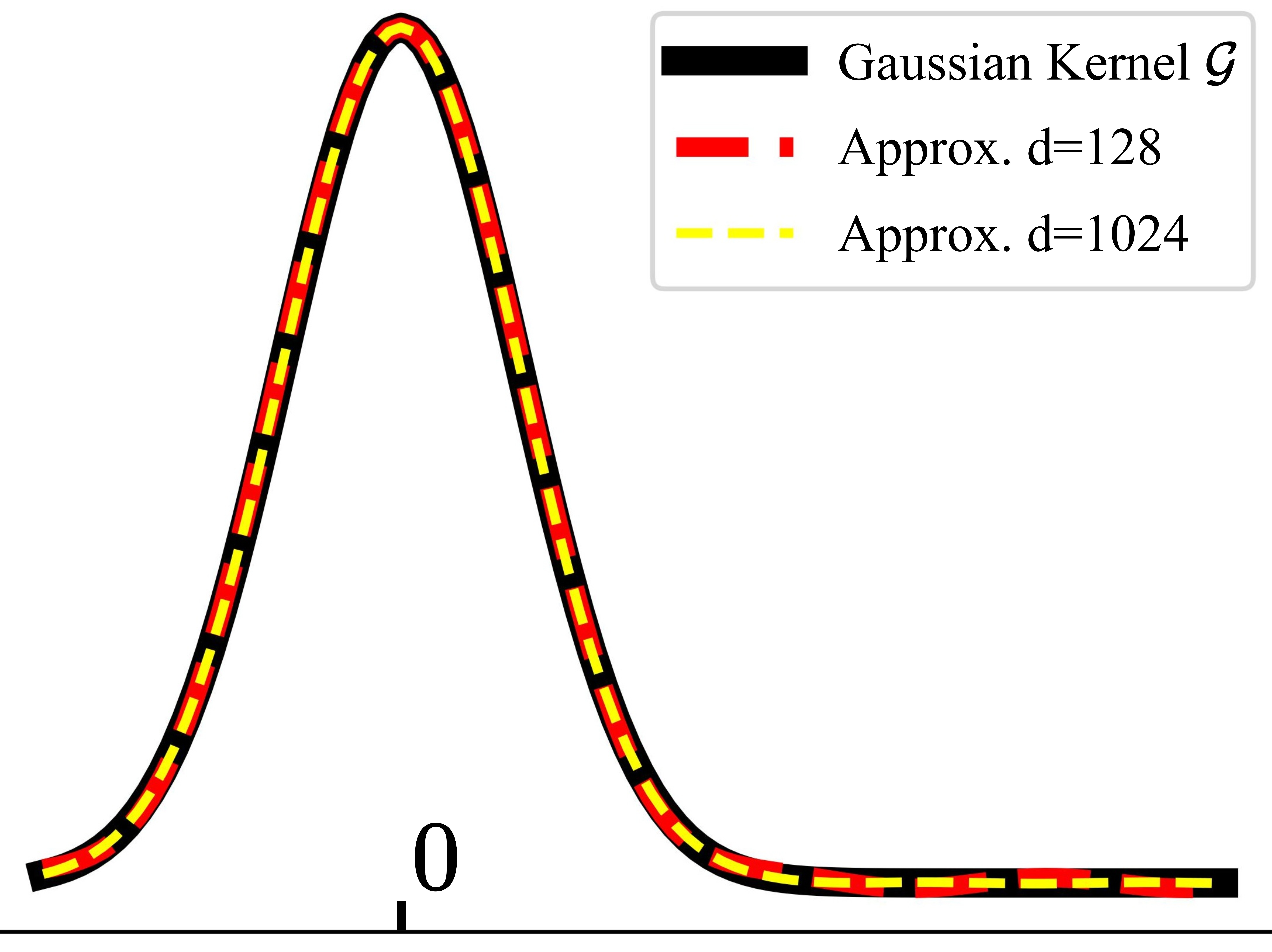}
  \vspace{-18pt}
  \caption{Visualization of Gaussian kernel $\mathcal{G}$ and its approximation with Lemma \ref{thm:lemma1}.}
  \vspace{-7pt}
  \label{fig:kernel}
\end{wrapfigure}
% \vspace{-2pt}
\textbf{(ii) The Gaussian kernel mixture associated with the point cloud approximates the shape function.} This is derived from the one-class support vector machine (SVM). The input to the one-class SVM is a collection of points and a user-defined kernel function, where the Gaussian (a.k.a. radial basis function) kernel is a common choice. The output of the one-class SVM is a kernel mixture which estimates the distribution of the input point set. \citet{scholkopf1999support} proves that with an appropriately chosen parameter $\alpha^2$ (defined in Lemma \ref{thm:lemma1}), a Gaussian kernel mixture can approximate the distribution function. This validates the assertion that the kernel mixtures associated with VecKM can approximate the shape distribution function. Coupled with Lemma \ref{thm:reconstruction_lemma}, \ref{thm:similarity_preservation_lemma}, we prove Proposition \ref{thm:reconstruction}, \ref{thm:similarity_preserving}, which reveal that VecKM effectively represents the local geometry.

% \vspace{-3pt}
\subsection{Dense Local Geometry Encoder}
\label{sec:dense_local_geometry}
% \vspace{-3pt}
In the previous section, we explained why Eqn. \ref{eqn:LGE} well represents the underlying local shape. In this section, we introduce the unique factorizable property that enables efficient computation of the dense local geometry encoding.

The geometry encoding in Eqn. \eqref{eqn:LGE} can be factorized into:
% \vspace{-10pt}
\begin{align*}
    E_\A\big(\mathfrak{N}(\x_0)\big)&=\frac{1}{n}\sum_{k=1}^n \exp\big(i (\x_k-\x_0)\A_{3\times d}\big) \\
    &=\frac{1}{n}\Big[\sum_{k=1}^n \exp(i \x_k\A)\Big] \: ./ \: \exp(i \x_0\A)
\end{align*}

Under this observation, we can write the dense local geometry encoding in terms of matrix computation:
\begin{align}
\begin{split}
    \mathcal{A}_{n\times d} &= \exp(i\X_{n\times 3} \A_{3\times d}) \\
    \G_{n\times d}&= [\mathbf{J}_{n\times n}\mathcal{A}_{n\times d}]\: ./ \: \mathcal{A}_{n\times d}
\end{split}
\label{eqn:proof_DLGE}
\end{align}
$\J_{n\times n}$ is the adjacency matrix of the point cloud $\X_{n\times 3}$, where $\J[j,k]=1$ if $||\x_j-\x_k||<r$ and 0 otherwise. Under this formulation, we still require $n^2$ time and space to compute the adjacency matrix $\J$ and $(n^2d)$ FLOPs to compute $\G$. But one important idea can be applied to speed up the computation: Instead of adopting a sharp threshold $r$ to define the adjacency relation, we employ an exponential decay function to establish this relationship:
\vspace{-1pt}
\[\hat{\J}[j,k]=\exp(-\beta^2||\x_j-\x_k||^2 / 2)\]
\vspace{-20pt}

where $\hat{\J}[j,k]$ decays from 1 to 0 as $||\x_j-\x_k||$ increases and the parameter $\beta$ controls the speed of decaying. As comparison, $\J[j,k]$ drops sharply from 1 to 0 when $||\x_j-\x_k||$ reaches $>r$. The parameter $\beta$ in $\hat{\J}$ has the same functionality as the parameter $r$ in $\J$, which is controlling the receptive field of the local neighbors. Arguably, $\J$ and $\hat{\J}$ behave similarly and it is natural to substitute $\J$ with $\hat{\J}$ in Eqn. \eqref{eqn:proof_DLGE}. The motivation of this substitution is that $\hat{\J}$ can be factorized into a matrix multiplication:
\vspace{-1pt}
\begin{align*}
    \mathcal{B}_{n\times p}&=\exp(i\X_{n\times 3}\B_{3\times p}) \\
    \hat{\J}_{n\times n} &\leftarrow \mathcal{B}\times \mathcal{B}^H \:\text{ as $p\to\infty$ }
\end{align*}
\vspace{-20pt}

where all entries in  $\B\in\mathbb{R}^{3\times p}$ follow $\mathcal{N}(0, \beta^2)$. Such approximation is, again, guaranteed by Lemma \ref{thm:lemma1}. With such approximation, Eqn. \eqref{eqn:proof_DLGE} can be rewritten as 
\begin{align*}
    \G_{n\times d}&= [\hat{\J}_{n\times n}\mathcal{A}_{n\times d}]\: ./ \: \mathcal{A}_{n\times d} \\
     &\approx [\mathcal{B}_{n\times p} \times (\mathcal{B}^H \times \mathcal{A})_{p\times d}] \: ./ \: \mathcal{A}_{n\times d}
\end{align*}
By computing $\mathcal{B}^H\times \mathcal{A}$ first, the computation cost is reduced to $\Theta(npd)$. A large point cloud size usually requires a larger $p$ to reduce the noise, but the value $p$ is much smaller than $n$. For a point cloud with size 100k, $p= 4096$ is sufficient. A large $p$ improves the quality of the encoding, but does not increase the size of the encoding, and hence does not increase the cost of subsequent processings. Such approximation-and-factorization trick is inspired from \citet{peng2021random}, which accelerates the attention computation in transformers. This concludes the proof of Theorem \ref{thm:DLGE}.

\textbf{Effect of $\alpha$ and $\beta$.} We perform a qualitative analysis of the effect of the parameters $\alpha$ and $\beta$ in Theorem \ref{thm:DLGE}. In short, $\alpha$ controls the level of details and $\beta$ controls the receptive field of the local neighbor. As illustrated in Figure \ref{fig:parameters}, when $\alpha$ is larger, more high-frequency details are preserved in the encoding, and meanwhile the local geometry encodings tends to be dissimilar to each other. A larger $\alpha$ is usually preferred in tasks that require refined local geometry, such as normal estimation. A smaller $\alpha$ is usually preferred in high-level tasks, such as classification and segmentation. For $\beta$ selection, we provide a table in Appendix \ref{sec:radius_beta}, which shows a matching between the neighborhood radius and the corresponding $\beta$ value. More quantitative analysis will be presented in Section \ref{sec:ablation}.  
% In practice, users can adopt multiple $\alpha$ and $\beta$ values, which may yield better performance than the single-scaled VecKM.

% \begin{figure}[H]
%     \centering
%     \vspace{-3pt}
%     \includegraphics[width=0.46\textwidth]{figs/figure3.drawio.pdf}
%     \vspace{-4pt}
%     \caption{Effect of the parameters $\alpha$ and $\beta$ in Theorem 
%     \ref{thm:DLGE}.}
%     \label{fig:parameters}
%     \vspace{-8pt}
% \end{figure}

\textbf{Effect of $d$ and $p$.} The parameters $d$ and $p$ control the quality of the encoding. Higher values lead to better quality of encoding. Figure \ref{fig:dp} provides the qualitative analysis.

% \begin{figure}[H]
%     \centering
%     \vspace{-3pt}
%     \includegraphics[width=0.46\textwidth]{figs/figure4.drawio.pdf}
%     \vspace{-4pt}
%     \caption{Effect of the parameters $d$ and $p$ in Theorem 
%     \ref{thm:DLGE}.}
%     \label{fig:dp}
%     \vspace{-8pt}
% \end{figure}

\textbf{Uniqueness of VecKM.} VecKM cannot be established without two important properties: 1. VecKM embodies a kernel function (Lemma \ref{thm:lemma1}); 2. VecKM is factorizable. Importantly, the family of exponential functions is the only family of functions that has the factorizability property with respect to multiplication and division: $f(x-y)=f(x)/f(y)$. But if we use the real exponential functions, the computation is not numerically stable, and meanwhile, the inner product between the constructed vectors will not induce a kernel, i.e. Lemma \ref{thm:lemma1} will not hold. Therefore, VecKM is the only choice to enable both properties, i.e. both being factorizable and inducing a kernel function. \emph{Therefore, we conjecture that VecKM may be the only possible linear local geometry encoder.} Fortunately, we are blessed with the advantages offered by complex vectors, which provide the necessary descriptiveness and efficiency for VecKM.

\vspace{-1pt}
\subsection{VecKM in Point Cloud Deep Learning}
\label{sec:VecKM_deep}
VecKM can seamlessly be integrated into widely-used deep point cloud architectures, including PointNet \citep{qi2017pointnet}, PointNet++ \citep{qi2017pointnet++}, and transformers \citep{guo2021pct, zhao2021point}. Typically, these architectures compute the dense local geometry in the first layer, often utilizing mini-PointNet or sequences of KPConvs \citep{thomas2019kpconv}. To use VecKM in those architectures, we simply replace the dense local geometry modules with our VecKM encodings.

Note that VecKM produces complex vector outputs. To effectively utilize this in subsequent layers, we employ a series of complex linear layers and complex ReLU layers \citep{DBLP:journals/corr/TrabelsiBSSSMRB17} to process the encodings. Finally, we cast the complex vectors into real vectors by calculating the squared norm of the complex vectors, thereby making the output compatible with standard architecture requirements. Figure \ref{fig:deep_VecKM} presents several examples of integrating VecKM into deep point cloud architectures, which are capable of solving many tasks involving point cloud inputs. Appendix \ref{sec:app_pytorch} gives the elegant implementation of VecKM in PyTorch.
% with fewer than 30 lines of codes.

\begin{figure}[t]
    \centering
    \includegraphics[width=0.48\textwidth]{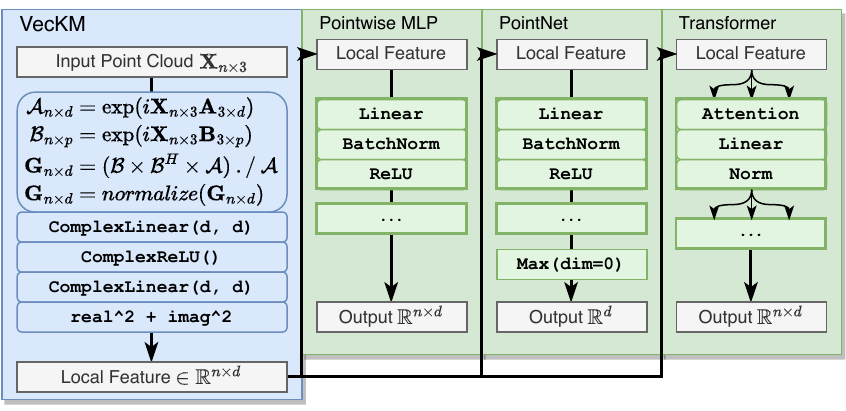}
    \vspace{-22pt}
    \caption{VecKM can be seamlessly integrated into deep point cloud architectures, improving both accuracy and efficiency.}
    \vspace{-12pt}
    \label{fig:deep_VecKM}
\end{figure}

\vspace{-2pt}
\section{Experiments}
We present extensive experiments to evaluate our VecKM encoding. In Section \ref{sec:normal_estimation}, we present quantitative and qualitative analyses on the effectiveness, efficiency, robustness, and scalability of the proposed VecKM encoding by solving the low-level task of normal estimation. In Section \ref{sec:classification}-\ref{sec:semantic_segmentation}, we demonstrate the effectiveness and efficiency when incorporating VecKM into deep point cloud architectures to solve high-level tasks. Depending on the input point cloud size, we use different implementations (either Eqn. \eqref{eqn:DLGE} or Eqn. \eqref{eqn:proof_DLGE}) of VecKM to yield the better efficiency.

\subsection{Normal Estimation on PCPNet Dataset}
\label{sec:normal_estimation}
We compare our VecKM against other local geometry encoders in four dimensions: accuracy, computational cost, memory cost, and robustness to noise. We select local point cloud normal estimation as our evaluation task because of its inherent challenges. This task requires the geometry encoders to adequately understand the local geometry. Moreover, it presents significant challenges in terms of memory and time complexity, given the large number of points in the input and the large number of neighboring points that need to be considered. As to be shown, \textbf{VecKM outperforms other encoders in all four dimensions by large margins.}

\textbf{Dataset and Metrics.} We use PCPNet \citep{guerrero2018pcpnet} as the evaluation dataset. PCPNet includes 8 shapes in the training set and 19 shapes in the test set. Each shape is sampled with 100,000 points and their ground-truth normals are derived from the original meshes. PCPNet provides two types of data corruption for testing: (1) point perturbations: adding Gaussian noise to the point coordinates. (2) point density variation: resampling the point cloud under two scenarios, where \textit{gradient} simulates the effects of varying distances from a sensor and \textit{strips} simulates the occlusion effect. We use the root mean squared angle error (RMSE) in degrees as the evaluation metrics.

% \vspace{-2pt}
\textbf{Compared Encoders.} We compare our VecKM against several widely-used local geometry encoders: PointNet \cite{qi2017pointnet}, KPConv \cite{thomas2019kpconv} and DGCNN \cite{wang2019dynamic}. 
% The implementation of each encoder is detailed in Appendix \ref{sec:app_normal_estimation_architectures}. 
\underline{\textbf{PointNet}}. The input point cloud is first grouped into the shape of $(n, K, 3)$ and transformed into the shape of $(n, K, d)$ by multi-layer perceptrons. Finally, a maxpooling operation shapes the data into $(n, d)$. $K$ is the number of neighboring points, which we attempt different values.
\underline{\textbf{KPConv}}. KPConv convolutes the local neighbors through a set of kernel points and transforms the convoluted features through a fully-connected layer. KPConv has a tunable parameter: the number of kernel points, which we attempt different values. 
\underline{\textbf{DGCNN}} models the neighboring points as dynamic graphs and performs edge convolution to aggregate the local feature. We adopt the architecture in the original paper, which consists of five layers of edge convolution. DGCNN has a tunable parameter: the number of neighbors being convoluted, which we attempt different values.
\underline{\textbf{VecKM}} (Ours): We adopt a multi-scale of $\alpha=60$ and $\beta=[10,20]$. Since the size of the point cloud is large, we implement VecKM by Eqn. \eqref{eqn:DLGE}. We set $d$ as 256 and $p$ as 4096.
We ensure the number of neighboring points considered by each encoder to be within 500$\sim$1000, which is sufficient to estimate the local normals. After encoding the local geometry, three layers of neural network are applied to predict the normals.

% \vspace{-2pt}
\textbf{Training Details.} Each model is trained with a batch size of 200 for a total of 200 epochs. We use the Adam optimizer, setting the learning rate at $10^{-3}$.  For data augmentation, Gaussian noise is added to the input point cloud. The input point cloud and their normals are randomly rotated. 

\begin{table}[t]
\centering
\caption{Normal estimation RMSE on the PCPNet dataset. }
\vspace{1pt}
\label{tab:normal_estimation}
\resizebox{0.48\textwidth}{!}{%
\begin{tabular}{lccccccc}
\toprule
 & \multicolumn{4}{c|}{Perturbations} & \multicolumn{2}{l|}{Density Variation} & \multirow{2}{*}{Average} \\ \cline{2-7}
 & \multicolumn{1}{l}{None} & \multicolumn{1}{l}{Low} & \multicolumn{1}{l}{Med} & \multicolumn{1}{l|}{High} & \multicolumn{1}{l}{Gradient} & \multicolumn{1}{l|}{Stripe} &  \\ \midrule
KPConv, \#kp=16 & 22.68 & 23.09 & 25.21 & 29.05 & 34.40 & 25.61 & 26.67 \\
KPConv, \#kp=32 & 22.74 & 22.21 & 24.08 & 28.25 & 32.24 & 24.94 & 25.74 \\
KPConv, \#kp=64 & 22.09 & 22.12 & 23.90 & 28.45 & 28.60 & 24.05 & 24.86 \\ \midrule
DGCNN, \#nbr=32 & 24.08 & 24.04 & 25.19 & 28.24 & 27.12 & 27.55 & 26.03 \\
DGCNN, \#nbr=64 & 23.21 & 25.34 & 25.66 & 26.01 & 28.86 & 28.20 & 26.21 \\
DGCNN, \#nbr=128 & 18.46 & 18.71 & 20.38 & 25.62 & 23.01 & 21.29 & 21.24 \\ \midrule
PointNet, \#nbr=300 & 14.98 & 16.30 & 20.19 & 26.83 & 23.68 & 19.00 & 20.17 \\
PointNet, \#nbr=500 & 16.10 & 16.54 & 21.38 & 26.93 & 26.06 & 18.89 & 20.99 \\
PointNet, \#nbr=700 & 15.59 & 16.25 & 20.99 & 26.21 & 24.66 & 17.87 & 20.27 \\
\midrule
VecKM (Ours) & \textbf{13.59} & \textbf{13.99} & \textbf{18.04} & \textbf{22.21} & \textbf{18.98} & \textbf{17.20} & \textbf{17.34} \\
\bottomrule
\end{tabular}%
}
\vspace{-10pt}
\end{table}

\begin{figure}[t]
    \centering
    \includegraphics[width=0.48\textwidth]{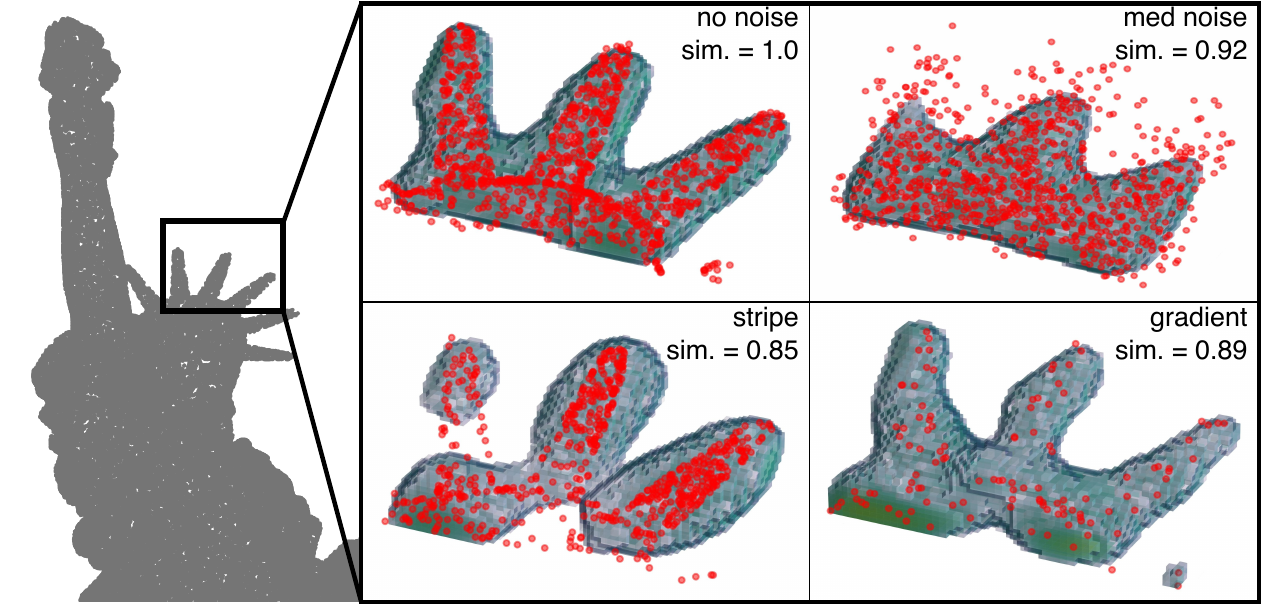}
    \vspace{-16pt}
    \caption{VecKM's robustness to data corruptions. VecKM can reconstruct the local shape under corrupted inputs. The VecKM encodings remain highly similar under data corruptions.}
    \label{fig:normal_explain}
    \vspace{-5pt}
\end{figure}

% \vspace{-2pt}
As shown in Table \ref{tab:normal_estimation}, \textbf{VecKM achieves $>16\%$ lower errors than all the compared encoders and performs the best under all data corruption settings.} This reveals that VecKM effectively captures the local geometry and is more robust to input perturbation and density variation. 
The effectiveness of VecKM can be attributed to its reconstructive and isometric properties, and its noise robustness is derived from the robustness inherent in the kernel mixture.
Figure \ref{fig:normal_explain} visualizes the explanation, which shows that even under corruptions, VecKM can still reconstruct local shapes and the VecKM encodings are consistent. 
In the case of the \textit{stripe} corruption setting, while the reconstruction may appear less accurate, the downstream neural network compensates for this discrepancy. This is evidenced by the relatively stable RMSE of the \textit{stripe} setting in Table \ref{tab:normal_estimation}, indicating that the overall impact on performance is not substantial.

% \vspace{-2pt}
As shown in Figure \ref{fig:exp_runtime}, \textbf{VecKM is $>$ 100x faster than all the compared encoders and is scalable to large point cloud inputs.} Even when the input size is as large as 100k, VecKM only takes $150$ ms to run. For memory cost, PointNet and DGCNN easily incur memory outrage when the neighbor size $K$ is large because they require an intermediate step of $(n, K, d)$ to compute the encoding. KPConv can be memory efficient through careful parallel programming, but existing implementations are not scalable to the settings we experiment with. VecKM, however, thanks to its unique factorizable property, only costs less than 8GB memory even with pure PyTorch implementation.

\vspace{-7pt}
\begin{figure}[H]
    \centering
    \includegraphics[width=0.5\textwidth]{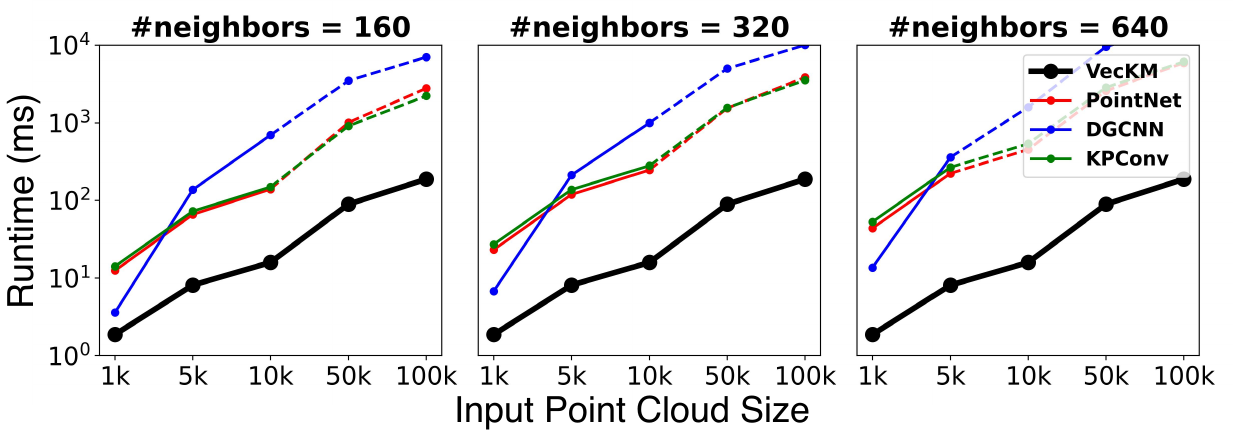}
    \vspace{-19pt}
    \caption{Runtime of local geometry encoders under different input point cloud size and neighbor size. All models are tested on an RTXA-5000 with 24 GB memory. Dash lines mean the memory is not sufficient to process all the points in one batch and has to process the points batches by batches.}
    \label{fig:exp_runtime}
    \vspace{-8pt}
\end{figure}

\subsection{Classification on ModelNet40 Dataset}
\label{sec:classification}
We evaluate our VecKM on 3D object classification using the ModelNet40 dataset \citep{wu20153d}. We compare classification accuracy and inference time with the baselines.

\textbf{Training Details.} We use the same training setting for all the methods. We use the official split with 9,843 objects for training and 2,468 for testing. Each point cloud is uniformly sampled to 1,024 points. During training, random translation in $[-0.2,0.2]$, and random scaling in $[0.67,1.50]$ are applied. We set the batch size to 32 and train the models for 250 epochs. We use the Adam optimizer, setting the initial learning rate as 0.001, with a cosine annealing scheduler. All models are trained and tested on an RTXA-5000.

\textbf{Baselines.} For our experiments, we select three widely-used point cloud architectures: PointNet \citep{qi2017pointnet}, PointNet++ \citep{qi2017pointnet++} and the Point Cloud Transformer (PCT) \citep{guo2021pct}. We integrate VecKM encoding into these architectures as outlined in Sec. \ref{sec:VecKM_deep}, which involves adding or replacing the original local geometry encoding modules with VecKM with $\alpha=30$, $\beta=6$ and $d=256$. Since the size of the point cloud is small, we implement VecKM by Eqn. \eqref{eqn:proof_DLGE}. We also compare VecKM-based architectures with the another light-weight network PointMLP \cite{ma2022rethinking}.

Specifically, for PointNet, since it does not have a local geometry encoding module, we \emph{add} the VecKM module before the PointNet, which means the PointNet receives the geometry encoding as input instead of the raw point coordinates. Since we \emph{add} (denoted by $\rightarrow$) VecKM as an additional module, the runtime is going to be longer. For PointNet++, we \emph{replace} (denoted by $\leftrightharpoons$) the first set abstraction layer with our VecKM encoding and leave the rest unchanged. For PCT, we \emph{replace} the initial input embedding module with the VecKM while retaining the transformer modules. For PointNet++ and PCT, since we \emph{replace} the dense local geometry encoding module with the more efficient VecKM, the runtime is expected to decrease.

As demonstrated in Table \ref{tab:classification}, \textbf{architectures based on VecKM consistently outperform their baseline counterparts in accuracy while also benefiting from significantly reduced runtime.} When VecKM is integrated with PointNet++ and PCT, not only is performance enhanced, but the speed of operation is also faster compared to the baselines. When comparing VecKM $\rightarrow$ PN against PointNet, there is a notable improvement in accuracy by 2.1\% and 2.6\%, with only a minimal increase in runtime. This is significant since the VecKM $\rightarrow$ PN architecture exhibits superior performance compared to both PointNet++ and PCT, and meanwhile operating 7.18x and 9.5x faster, respectively. Compared with PointMLP, VecKM-based architectures are even more efficient, achieving on-par accuracies.

% As shown in Table \ref{tab:classification}, \textbf{VecKM-based architectures achieve consistently better accuracy compared with their baselines with huge runtime privilege.} When incorporating VecKM with PointNet++ and PCT, the VecKM-based architectures achieve better performance, and meanwhile runs faster than their baselines. Comparing VecKM -- PN and PointNet, the accuracy is improved significantly by 2.1\% and 2.6\% with very marginal overhead. This can be seen by comparing VecKM -- PN with PointNet++ and PCT: the VecKM -- PN architecture achieves better performance than PointNet++ and PCT, while runs 7.18x and 9.5x faster.

% Comparing with the PointNet++ baseline, VecKM, as a pre-processing module, improves the performance significantly at only 0.19\% computational overhead. Comparing with PointNet and PCT baselines, replacing the local geometry encoding module with VecKM not only yields better performance, but also accelerates the inference by 1.28x and 41x.

% Please add the following required packages to your document preamble:
% \usepackage{booktabs}
% \usepackage{graphicx}
% \usepackage[table,xcdraw]{xcolor}
% Beamer presentation requires \usepackage{colortbl} instead of \usepackage[table,xcdraw]{xcolor}
\begin{table}[t]
\centering
\caption{Classification performance on the ModelNet40 dataset. VecKM $\rightarrow$ PN means \emph{adding} VecKM as a preprocessing module to PointNet, so the runtime is expected to be longer than the PointNet baseline. VecKM $\leftrightharpoons$ PN++/PCT means \emph{replacing} the original dense local geometry encoding in the original architectures with VecKM. Since VecKM is more efficient, the runtime is reduced.}
\label{tab:classification}
\resizebox{0.49\textwidth}{!}{%
\begin{tabular}{@{}lcccc@{}}
\toprule
 &
  \begin{tabular}[c]{@{}c@{}}Instance\\ Accuracy\end{tabular} &
  \begin{tabular}[c]{@{}c@{}}Avg. Class\\ Accuracy\end{tabular} &
  \begin{tabular}[c]{@{}c@{}}Inference Time (ms)\\ (1 batch)\end{tabular} &
  \# parameters \\ \midrule
PointMLP &
  93.2\% &
  90.1\% &
  325.85 &
  13.2M \\ \midrule
PointNet &
  90.8\% &
  87.1\% &
  3.04 &
  1.61M \\
VecKM $\rightarrow$ PN &
  92.9\% &
  89.7\% &
  14.32 &
  9.06M \\
\rowcolor[HTML]{EFEFEF} 
Difference &
  {\color[HTML]{FE0000} $\uparrow$ 2.1\%} &
  {\color[HTML]{FE0000} $\uparrow$ 2.6\%} &
  {\color[HTML]{000000} not comparable} &
  {\color[HTML]{32CB00} +7.61M} \\ \midrule
PointNet++ &
  92.7\% &
  89.4\% &
  117.13 &
  1.48M \\
VecKM $\leftrightharpoons$ PN++ &
  93.0\% &
  89.7\% &
  65.78 &
  3.94M \\
\rowcolor[HTML]{EFEFEF} 
Difference &
  {\color[HTML]{FE0000} $\uparrow$ 0.3\%} &
  {\color[HTML]{FE0000} $\uparrow$ 0.3\%} &
  {\color[HTML]{FE0000} 78\% faster} &
  {\color[HTML]{32CB00} +2.46M} \\ \midrule
PCT &
  92.9\% &
  89.8\% &
  149.72 &
  2.88M \\
VecKM $\leftrightharpoons$ PCT &
  93.1\% &
  90.6\% &
  21.44 &
  5.07M \\
\rowcolor[HTML]{EFEFEF} 
Difference &
  {\color[HTML]{FE0000} $\uparrow$ 0.2\%} &
  {\color[HTML]{FE0000} $\uparrow$ 0.8\%} &
  {\color[HTML]{FE0000} 5.98x faster} &
  {\color[HTML]{32CB00} +2.19M} \\ \bottomrule
\end{tabular}%
}
\vspace{-12pt}
\end{table}

% Please add the following required packages to your document preamble:
% \usepackage{booktabs}
% \usepackage{graphicx}
% \usepackage[table,xcdraw]{xcolor}
% Beamer presentation requires \usepackage{colortbl} instead of \usepackage[table,xcdraw]{xcolor}
\begin{table}[t]
\centering
\caption{Part segmentation performance on the ShapeNet dataset. Similar to the classification, $\rightarrow$ means adding VecKM as a preprocessing module, so the runtime is expected to be longer. $\leftrightharpoons$ means replacing the dense local geometry encoding module with VecKM. Since VecKM is more efficient, the runtime is reduced.}
\label{tab:part_segmentation}
\resizebox{0.49\textwidth}{!}{%
\begin{tabular}{@{}lcccc@{}}
\toprule
 &
  \begin{tabular}[c]{@{}c@{}}Instance\\ mIoU\end{tabular} &
  \begin{tabular}[c]{@{}c@{}}Avg. Class\\ mIoU\end{tabular} &
  \begin{tabular}[c]{@{}c@{}}Inference Time (ms)\\ (1 batch)\end{tabular} &
  \# parameters \\ \midrule
PointMLP                        & 85.1\% & 82.1\% & 240.39 & 16.76M \\ \midrule
PointNet                        & 83.1\% & 77.6\% & 15.1   & 8.34M  \\
VecKM $\rightarrow$ PN          & 84.9\% & 81.8\% & 40.8   & 1.29M  \\
\rowcolor[HTML]{EFEFEF} 
\cellcolor[HTML]{EFEFEF}Difference &
  {\color[HTML]{FE0000} $\uparrow$ 1.8\%} &
  {\color[HTML]{FE0000} $\uparrow$ 4.2\%} &
  {\color[HTML]{000000} not comparable} &
  {\color[HTML]{32CB00} +7.05M} \\ \midrule
PointNet++                      & 85.0\% & 81.9\% & 130.8  & 1.41M  \\
VecKM $\leftrightharpoons$ PN++ & 85.3\% & 82.0\% & 65.9   & 1.50M  \\
\rowcolor[HTML]{EFEFEF} 
\cellcolor[HTML]{EFEFEF}Difference &
  {\color[HTML]{FE0000} $\uparrow$ 0.3\%} &
  {\color[HTML]{FE0000} $\uparrow$ 0.1\%} &
  {\color[HTML]{FE0000} 98\% faster} &
  {\color[HTML]{32CB00} +0.09M} \\ \midrule
PCT                             & 85.7\% & 82.6\% & 145.2  & 1.63M  \\
VecKM $\leftrightharpoons$ PCT  & 85.8\% & 82.6\% & 46.6   & 1.71M  \\
\rowcolor[HTML]{EFEFEF} 
\cellcolor[HTML]{EFEFEF}Difference &
  {\color[HTML]{FE0000} $\uparrow$ 0.1\%} &
  {\color[HTML]{000000} 0.0\%} &
  {\color[HTML]{FE0000} 2.11x faster} &
  {\color[HTML]{32CB00} +0.08M} \\ \bottomrule
\end{tabular}%
}
\vspace{-12pt}
\end{table}

\subsection{Part Segmentation on ShapeNet Dataset}
\label{sec:part_segmentation}
We evaluate our VecKM on 3D object part segmentation. Our experiment utilizes the ShapeNet \citep{chang2015shapenet} dataset. Similar to the classification experiment, we compare the IoU and inference time with PointNet, PointNet++, PCT, and PointMLP. The baselines and their VecKM counter-parts are obtained like the classification experiment in Section \ref{sec:classification}. The parameters of VecKM are selected as $\alpha=30, \beta=9$ and $d=256$.
Since the size of the point cloud is small, we implement VecKM by Eqn. \eqref{eqn:proof_DLGE}. 

\textbf{Training Details.} We use the same training setting for all the methods. We use the official split with 14,006 3D models for training and 2,874 for testing. Each point cloud is uniformly sampled to 2,048 points. During training, random translation in $[-0.2, 0.2]$, and random scaling in $[0.67, 1.50]$ are applied. We set the batch size to 16 and train the model for 250 epochs. We use the Adam optimizer, setting the initial learning rate as 0.001, with a cosine annealing scheduler. All models are trained and tested on an RTXA-5000.

As demonstrated in Table \ref{tab:part_segmentation}, similar to the classfication experiment, architectures based on VecKM consistently outperform their baseline counter-parts in accuracy while also benefiting from significantly reduced runtime.

\subsection{Semantic Segmentation on S3DIS Dataset}
\label{sec:semantic_segmentation}
We evaluate our VecKM on 3D semantic segmentation. We use the S3DIS dataset \cite{armeni2017joint}, which is an indoor scene dataset. It contains 6 areas and 271 rooms. Each point in this dataset is classified into one of 13 categories. Each scene contains around 10,000$\sim$100,000 points. We use the same training setting as \citet{zhao2021point}.

\textbf{Baselines.} We select PointNet++ and Point Transformer \cite{zhao2021point} as the baselines. For \underline{PointNet++}, in its first layer, PointNet++ first downsamples the point cloud by $1/4$ and for each sampled point, $32$ neighboring points are sampled and transformed by a PointNet. The VecKM $\rightarrow$ PN++ counter-part is obtained by \emph{adding} the dense local geometry encoder before the first layer.  Consequently, the PointNet in the first layer will transform the local geometry encoding instead of the raw 3d coordinates. Because of the downsampling operation in PointNet++, its inference time is much shorter. Therefore, PN++ and VecKM $\rightarrow$ PN++ are not comparable in terms of inference time. For \underline{Point Transformer}, its first layer is a dense local geometry encoder with PointNet. We \emph{replace} the dense local geometry encoder with our VecKM encoding to obtain the PT $\leftrightharpoons$ VecKM architecture. In both architectures, since the size of the point cloud is large, we implement VecKM by Eqn. \eqref{eqn:DLGE}. We set $\alpha=30, \beta=9, d=256, p=2048$, and we use a sequence of two complex linear layers to transform the local geometry encoding from $\mathbb{C}^{256}$ to $\mathbb{C}^{64}$.

As shown in Table \ref{tab:sem_segmentation}, \textbf{VecKM improves PointNet++ baseline significantly}. This is because the downsampling of the point cloud induces information loss in the PointNet++ baseline, while the dense VecKM encoding effectively bridges the gap. On the other hand, \textbf{VecKM improves the inference speed of point transformer}, which is expected given the efficiency of VecKM especially on large point cloud input. Regarding why VecKM $\leftrightharpoons$ PT does not yield better accuracy, it is possibly because the heavy-weight point transformer architecture already adequately reasons on the geometry. Unlike PointNet++, the local geometry encoding is not a bottleneck for point transformer. Since the subsequent processing costs the majority of the running time, the acceleration is not as significant as the previous experiments.

% Please add the following required packages to your document preamble:
% \usepackage{booktabs}
% \usepackage{graphicx}
% \usepackage[table,xcdraw]{xcolor}
% Beamer presentation requires \usepackage{colortbl} instead of \usepackage[table,xcdraw]{xcolor}
\begin{table}[]
\centering
\caption{Semantic segmentation performance on the S3DIS dataset. Similar to the classification experiment, $\rightarrow$ means adding VecKM as a preprocessing module. Since PointNet++ downsamples the point cloud at the first layer while VecKM $\rightarrow$ PN++ does not, their inference time is not comparable. $\leftrightharpoons$ means replacing the original dense local geometry encoding module with VecKM. Since VecKM is more efficient, the runtime is reduced.}
\label{tab:sem_segmentation}
\resizebox{0.49\textwidth}{!}{%
\begin{tabular}{@{}lccccc@{}}
\toprule
 &
  \begin{tabular}[c]{@{}c@{}}Instance\\ mIoU\end{tabular} &
  \begin{tabular}[c]{@{}c@{}}Avg. Class\\ mIoU\end{tabular} &
  \begin{tabular}[c]{@{}c@{}}Overall\\ Accuracy\end{tabular} &
  \begin{tabular}[c]{@{}c@{}}Inference Time (ms)\\ (per scene)\end{tabular} &
  \multicolumn{1}{l}{\#parameters} \\ \midrule
PointNet++ &
  64.05 &
  71.52 &
  87.92 &
  96 &
  0.968M \\
VecKM $\rightarrow$ PN++ &
  67.48 &
  73.53 &
  89.33 &
  391 &
  1.11M \\
\rowcolor[HTML]{EFEFEF} 
Difference &
  {\color[HTML]{FE0000} $\uparrow$ 3.43} &
  {\color[HTML]{FE0000} $\uparrow$ 2.01} &
  {\color[HTML]{FE0000} $\uparrow$ 1.41} &
  {\color[HTML]{000000} not comparable} &
  {\color[HTML]{32CB00} +0.142M} \\ \midrule
Point Transformer &
  69.29 &
  75.66 &
  90.36 &
  559 &
  7.77M \\
VecKM $\leftrightharpoons$ PT &
  69.53 &
  75.84 &
  90.39 &
  447 &
  7.93M \\
\rowcolor[HTML]{EFEFEF} 
Difference &
  {\color[HTML]{FE0000} $\uparrow$ 0.24} &
  {\color[HTML]{FE0000} $\uparrow$ 0.18} &
  {\color[HTML]{FE0000} $\uparrow$ 0.03} &
  {\color[HTML]{FE0000} 20\% faster} &
  {\color[HTML]{32CB00} +0.16M} \\ \bottomrule
\end{tabular}%
}
\vspace{-14pt}
\end{table}

\section{Ablation Studies}
\label{sec:ablation}
In Section \ref{sec:dense_local_geometry}, we qualitatively analyze the effect of the parameters $\alpha$ and $\beta$ in Theorem \ref{thm:DLGE}. In this section, we quantitatively analyze the effect of the parameters in the context of the ModelNet40 classification experiment, with the VecKM $\rightarrow$ PN architecture. For $\alpha$ selection, when the input point cloud is normalized within a unit ball, setting $\alpha$ in the range of $(20,35)$ yields good performance. As shown in Table \ref{tab:ablation}, appropriate selections of $\alpha$ and $\beta$ are important to yield a good performance on the downstream tasks.
% Please add the following required packages to your document preamble:
% \usepackage{booktabs}
% \usepackage{graphicx}
\begin{table}[t]
\centering
\vspace{-8pt}
\caption{Ablation study on the selection of the parameters $\alpha$ and $\beta$ in Theorem \ref{thm:DLGE}, in the context of ModelNet40 classification experiment. Numbers greater than $92.5\%$ are bolded.}
\label{tab:ablation}
\resizebox{0.35\textwidth}{!}{%
\begin{tabular}{@{}l|rrrr@{}}
\toprule
 & \multicolumn{1}{l}{$\alpha=20$} & \multicolumn{1}{l}{$\alpha=25$} & \multicolumn{1}{l}{$\alpha=30$} & \multicolumn{1}{l}{$\alpha=35$} \\ \midrule
$\beta=4$  & 91.73\%          & 91.94\%          & 91.73\%          & 91.77\%          \\
$\beta=6$  & \textbf{92.59\%} & 92.14\%          & \textbf{92.87\%} & \textbf{92.50\%} \\
$\beta=9$  & 92.18\%          & \textbf{92.71\%} & \textbf{92.95\%} & \textbf{92.50\%} \\
$\beta=12$ & 92.10\%          & \textbf{92.54\%} & \textbf{92.59\%} & 92.38\%          \\ \bottomrule
\end{tabular}%
}
\vspace{-8pt}
\end{table}

\vspace{-1pt}
\begin{wrapfigure}{r}{0.22\textwidth}
  \centering
  \vspace{-14pt}
  \includegraphics[width=0.2\textwidth]{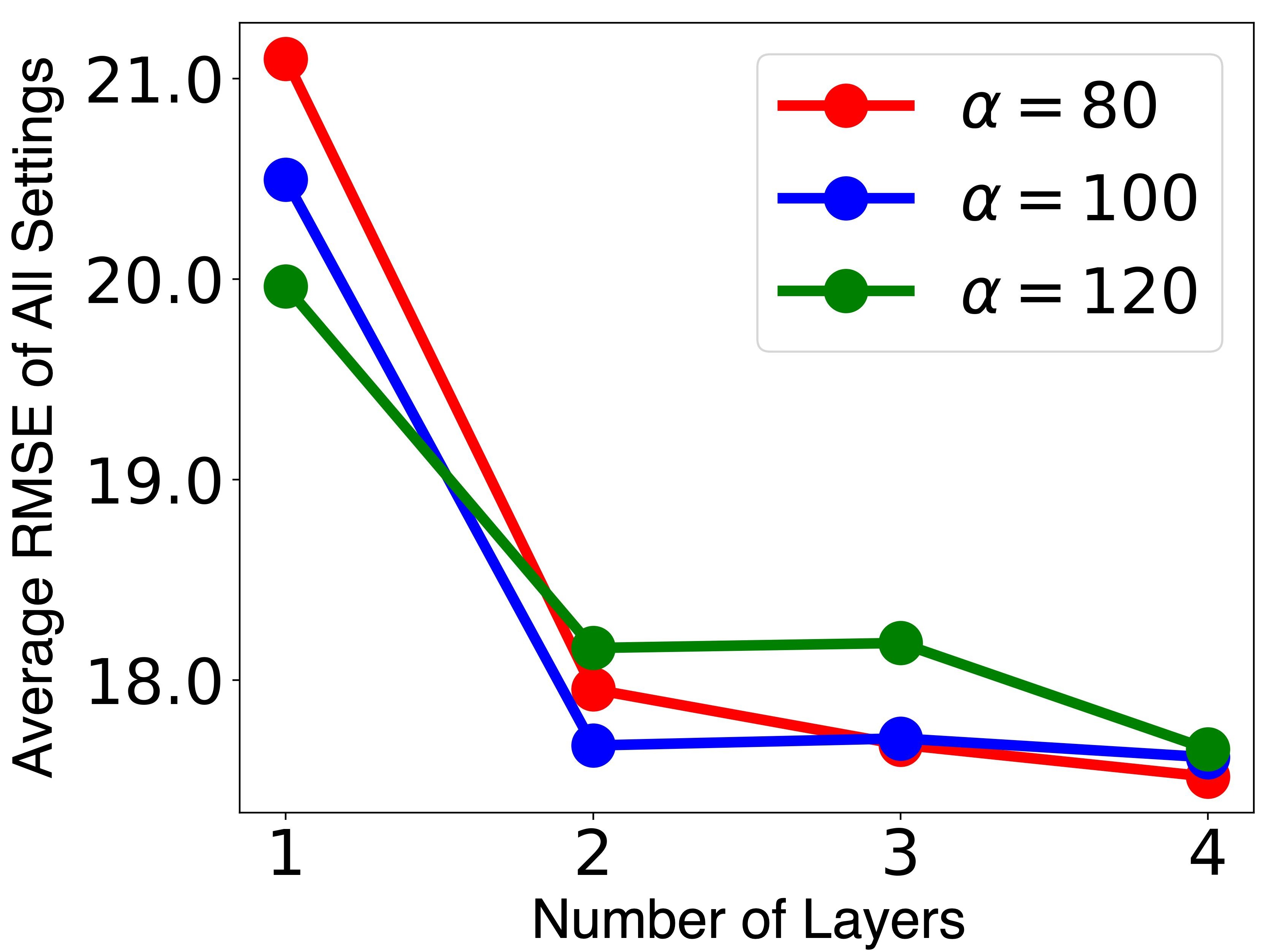}
  \vspace{-9pt}
  \caption{Average RMSE of normal estimation trained with different numbers of layers.}
  \vspace{-7pt}
  \label{fig:ablation_layers}
\end{wrapfigure}
We study how many fully-connected layers are needed for transforming the VecKM encoding, in the context of normal estimation tasks. As shown in Figure \ref{fig:ablation_layers}, two layers are sufficient for stably satisfactory performance, highlighting the inherent descriptiveness of VecKM encoding.

\vspace{-1pt}
Notice that the selection of the $\alpha$ parameter varies across different tasks. In classification, where refined local geometry is less critical, a smaller $\alpha$ is used to abstract away finer details. For normal estimation tasks, where accurate local shape representation is crucial, a larger $\alpha$ is employed to retain essential details. These findings demonstrate VecKM's adaptability in meeting the diverse requirements of various tasks, adjusting to the specific level of detail needed.

\vspace{-6pt}
\section{Conclusion}
VecKM, our novel local point cloud encoder, stands out for its efficiency and noise robustness. VecKM vectorizes a kernel mixture associated with the local point cloud, providing a solid theoretical foundation for its descriptiveness and robustness. Thanks to its special formulation, VecKM is the only existing local geometry encoder that costs linear time and space. Through extensive experiments, VecKM has demonstrated significant improvements in speed and accuracy across a variety of point cloud processing tasks. VecKM has many potential applications due to its notable features. Its efficiency facilitates faster inference, ideal for time-critical tasks like event data processing.
% Additionally, its robustness makes it an ideal candidate for local geometry encoding in point cloud foundation models.

\newpage

\section*{Impact Statements}
This paper presents work whose goal is to advance the field of Machine Learning. There are many potential societal consequences of our work, none which we feel must be specifically highlighted here.

\section*{Acknowledgement}
The support of NSF under awards OISE 2020624 and BCS 2318255, and ARL under the Army Cooperative Agreement W911NF2120076 is greatly acknowledged. We extend our gratitude to Pixabay for providing free access to the icon used in drafting Figure \ref{fig:proof_outline} and \ref{fig:parameters}.

% In the unusual situation where you want a paper to appear in the
% references without citing it in the main text, use \nocite
% \nocite{chen2022estimation}

\bibliography{example_paper}
\bibliographystyle{icml2024}

%%%%%%%%%%%%%%%%%%%%%%%%%%%%%%%%%%%%%%%%%%%%%%%%%%%%%%%%%%%%%%%%%%%%%%%%%%%%%%%
%%%%%%%%%%%%%%%%%%%%%%%%%%%%%%%%%%%%%%%%%%%%%%%%%%%%%%%%%%%%%%%%%%%%%%%%%%%%%%%
% APPENDIX
%%%%%%%%%%%%%%%%%%%%%%%%%%%%%%%%%%%%%%%%%%%%%%%%%%%%%%%%%%%%%%%%%%%%%%%%%%%%%%%
%%%%%%%%%%%%%%%%%%%%%%%%%%%%%%%%%%%%%%%%%%%%%%%%%%%%%%%%%%%%%%%%%%%%%%%%%%%%%%%
\newpage
\appendix
\onecolumn
\newpage
\appendix
\onecolumn
\section{Proof of Lemma 1}
\label{sec:app_lemma1}
\begin{lemma*}[VecKM embodies a Gaussian kernel]
Let $\x, \y\in\mathbb{R}^3$, $\A\in\mathbb{R}^{3\times d}$. All elements in $\A$ are drawn from normal distribution $\mathcal{N}(0, \alpha^2)$. Then as $d\to\infty$,
\begin{align*}
\begin{split}
    \frac{1}{d}\langle e^{i\x\A}, e^{i\y\A}\rangle\rightarrow \GG_\alpha(\x, \y):=\exp\big(-\frac{\alpha^2||\x-\y||^2}{2}\big)
\end{split}
\end{align*}
\end{lemma*}
\begin{proof}
    Let $\ba\in\mathbb{R}^3$ where $\ba\sim\mathcal{N}(\mathbf{0},\alpha^2 \mathbf{I}_{3\times 3})$ be one column of the matrix $\A$, we claim that $\mathbb{E}\big[\Re\big(e^{i\ba\cdot (\x-\y)}\big)\big]=\GG_\alpha(\x,\y)$:
    \begin{align*}
        \mathbb{E}\Big[\Re\big(e^{i \mathbf{\ba}\cdot (\x-\y)}\big)\Big]&=\mathbb{E}\Big[\cos\big(\ba\cdot(\x-\y)\big)\Big]\\
        &=\mathbb{E}\Big[\sum_{k=0}^\infty (-1)^k \frac{\big(\mathbf{\ba}\cdot(\x-\y)\big)^{2k}}{(2k)!}\Big] \\
        &=\sum_{k=0}^\infty \frac{(-1)^k}{(2k)!} \mathbb{E}\Big[\big(\sum_{j=1}^3 (\x_j-\y_j)\ba_j \big)^{2k}\Big] \\
        &=\sum_{k=0}^\infty \frac{(-1)^k}{(2k)!} \mathbb{E}\Big[\big(\alpha||\x-\y|| Z\big)^{2k}\Big] \quad \text{where $Z\in\mathcal{N}(0,1)$} \\
        &=\sum_{k=0}^\infty \frac{(-1)^k}{(2k)!} \cdot \alpha^{2k}||\x-\y||^{2k} \cdot \frac{(2k)!}{k! 2^k} \\
        &=\sum_{k=0}^\infty \frac{(-1)^k}{k! 2^k} \cdot \alpha^{2k}||\x-\y||^{2k} \\
        &=\exp\big(-\frac{\alpha^2||\x-\y||^2}{2}\big)
    \end{align*}

    On the other hand, $\mathbb{E}\big[\Im\big(e^{i\ba\cdot (\x-\y)}\big)\big]=0$ because normal distribution is a symmetric distribution around 0:
    \begin{align*}
        \mathbb{E}\Big[\Im\big(e^{i \mathbf{\ba}\cdot (\x-\y)}\big)\Big]&=\mathbb{E}\Big[\sin\big(\ba\cdot(\x-\y)\big)\Big]\\
        &=\mathbb{E}\Big[\sum_{k=0}^\infty (-1)^k \frac{\big(\mathbf{\ba}\cdot(\x-\y)\big)^{2k+1}}{(2k+1)!}\Big] \\
        &=\sum_{k=0}^\infty \frac{(-1)^k}{(2k+1)!} \mathbb{E}\Big[\big(\sum_{j=1}^3 (\x_j-\y_j)\ba_j \big)^{2k+1}\Big] \\
        &=\sum_{k=0}^\infty \frac{(-1)^k}{(2k+1)!} \mathbb{E}\Big[\big(\alpha||\x-\y|| Z\big)^{2k+1}\Big] \quad \text{where $Z\in\mathcal{N}(0,1)$} \\
        &=0 \quad \text{because $\mathbb{E}(Z^{2k+1})=0$}
    \end{align*}

    Therefore, when we randomize $d$ rows of such $\ba$ vector, the inner product $\frac{1}{d}\langle e^{i\x\A}, e^{i\y\A}\rangle = \frac{1}{d}\sum_{k=1}^d e^{i\ba_k\cdot(\x-\y)}$ will converge to $\GG_\alpha (\x, \y)$ thanks to the Law of Large Number and the Central Limit Theorem.
\end{proof}

\newpage
\section{PyTorch Implementation of VecKM}
\label{sec:app_pytorch}
\begin{lstlisting}
import torch
import torch.nn as nn
import numpy as np
from scipy.stats import norm

def strict_standard_normal(d):
    # this function generate very similar outcomes as torch.randn(d)
    # but the numbers are strictly standard normal, no randomness.
    y = np.linspace(0, 1, d+2)
    x = norm.ppf(y)[1:-1]
    np.random.shuffle(x)
    x = torch.tensor(x).float()
    return x

class VecKM(nn.Module):
    def __init__(self, d=256, alpha=6, beta=1.8, p=4096):
        super().__init__()
        self.alpha, self.beta, self.d, self.p = alpha, beta, d, p
        self.sqrt_d = d ** 0.5

        self.A = torch.stack(
            [strict_standard_normal(d) for _ in range(3)], 
            dim=0
        ) * alpha
        self.A = nn.Parameter(self.A, False)            # Real(3, d)

        self.B = torch.stack(
            [strict_standard_normal(p) for _ in range(3)], 
            dim=0
        ) * beta
        self.B = nn.Parameter(self.B, False)            # Real(3, d)

    def forward(self, pts):
        """ Compute the dense local geometry encodings of the given point cloud.
        Args:
            pts: (bs, n, 3) or (n, 3) tensor, the input point cloud.

        Returns:
            G: (bs, n, d) or (n, d) tensor. the dense local geometry encodings. 
        """
        pA = pts @ self.A                               # Real(..., n, d)
        pB = pts @ self.B                               # Real(..., n, p)
        eA = torch.concatenate(
            (torch.cos(pA), torch.sin(pA)), dim=-1)     # Real(..., n, 2d)
        eB = torch.concatenate(
            (torch.cos(pB), torch.sin(pB)), dim=-1)     # Real(..., n, 2p)
        G = torch.matmul(
            eB,                                         # Real(..., n, 2p)
            eB.transpose(-1,-2) @ eA                    # Real(..., 2p, 2d)
        )                                               # Real(..., n, 2d)
        G = torch.complex(
            G[..., :self.d], G[..., self.d:]
        ) / torch.complex(
            eA[..., :self.d], eA[..., self.d:]
        )                                               # Complex(..., n, d)
        G = G / torch.norm(G, dim=-1, keepdim=True) * self.sqrt_d
        return G

vkm = VecKM()
pts = torch.rand((10,1000,3))
print(vkm(pts).shape) # it will be Complex(10,1000,256)
pts = torch.rand((1000,3))
print(vkm(pts).shape) # it will be Complex(1000, 256)
from complexPyTorch.complexLayers import ComplexLinear, ComplexReLU
# You may want to use apply two-layer feature transform to the encoding.
feat_trans = nn.Sequential(
    ComplexLinear(256, 128),
    ComplexReLU(),
    ComplexLinear(128, 128)
)
G = feat_trans(vkm(pts))
G = G.real**2 + G.imag**2 # it will be Real(10, 1000, 128) or Real(1000, 128).
\end{lstlisting}

\newpage
\section{Effect of Parameters $\alpha, \beta, d, p$}
\label{sec:app_parameters}
\begin{figure}[H]
    \centering
    \vspace{-3pt}
    \includegraphics[width=0.7\textwidth]{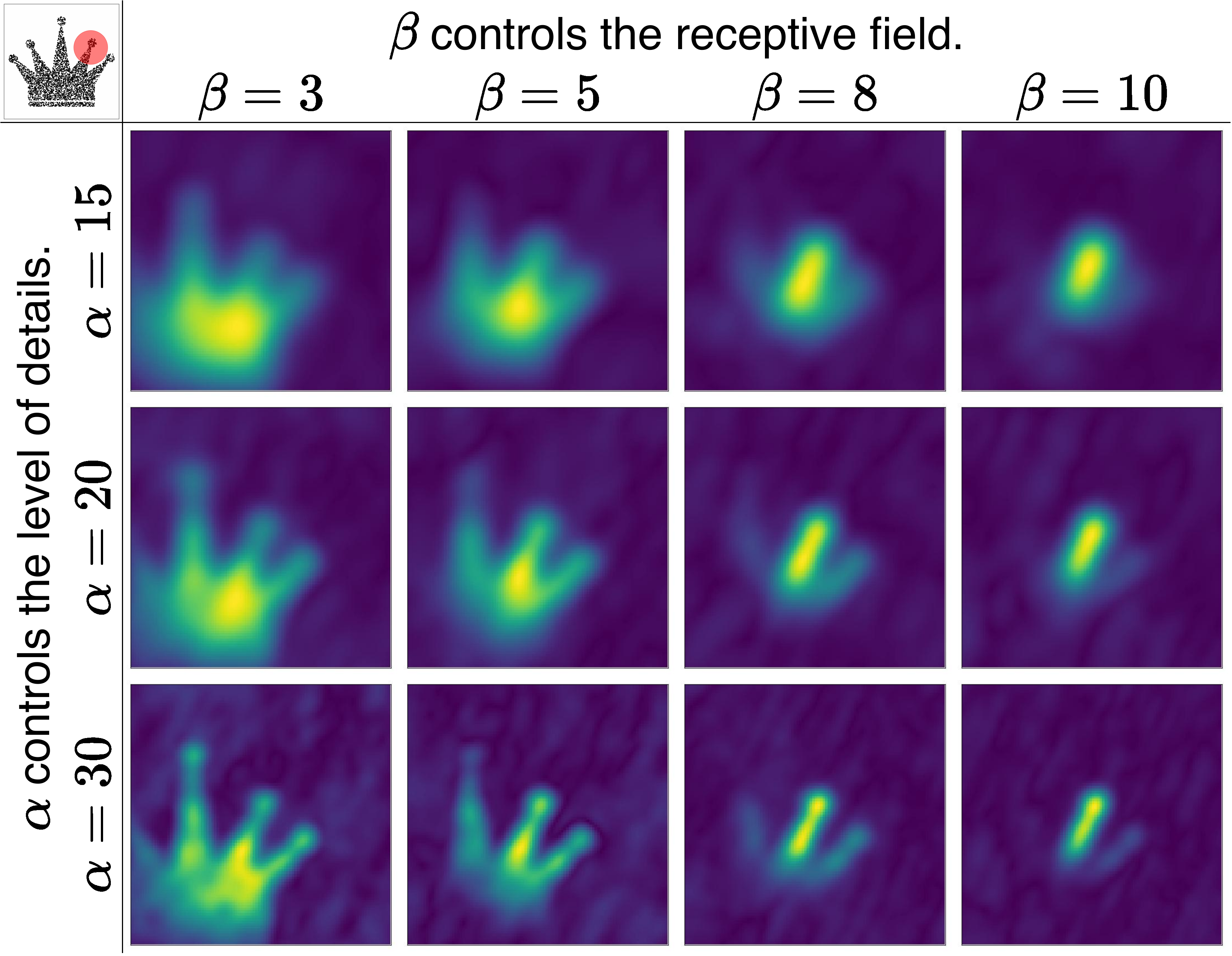}
    \vspace{-4pt}
    \caption{Effect of the parameters $\alpha$ and $\beta$ in Theorem 
    \ref{thm:DLGE}.}
    \label{fig:parameters}
    \vspace{-8pt}
\end{figure}

\begin{figure}[H]
    \centering
    \vspace{-3pt}
    \includegraphics[width=0.7\textwidth]{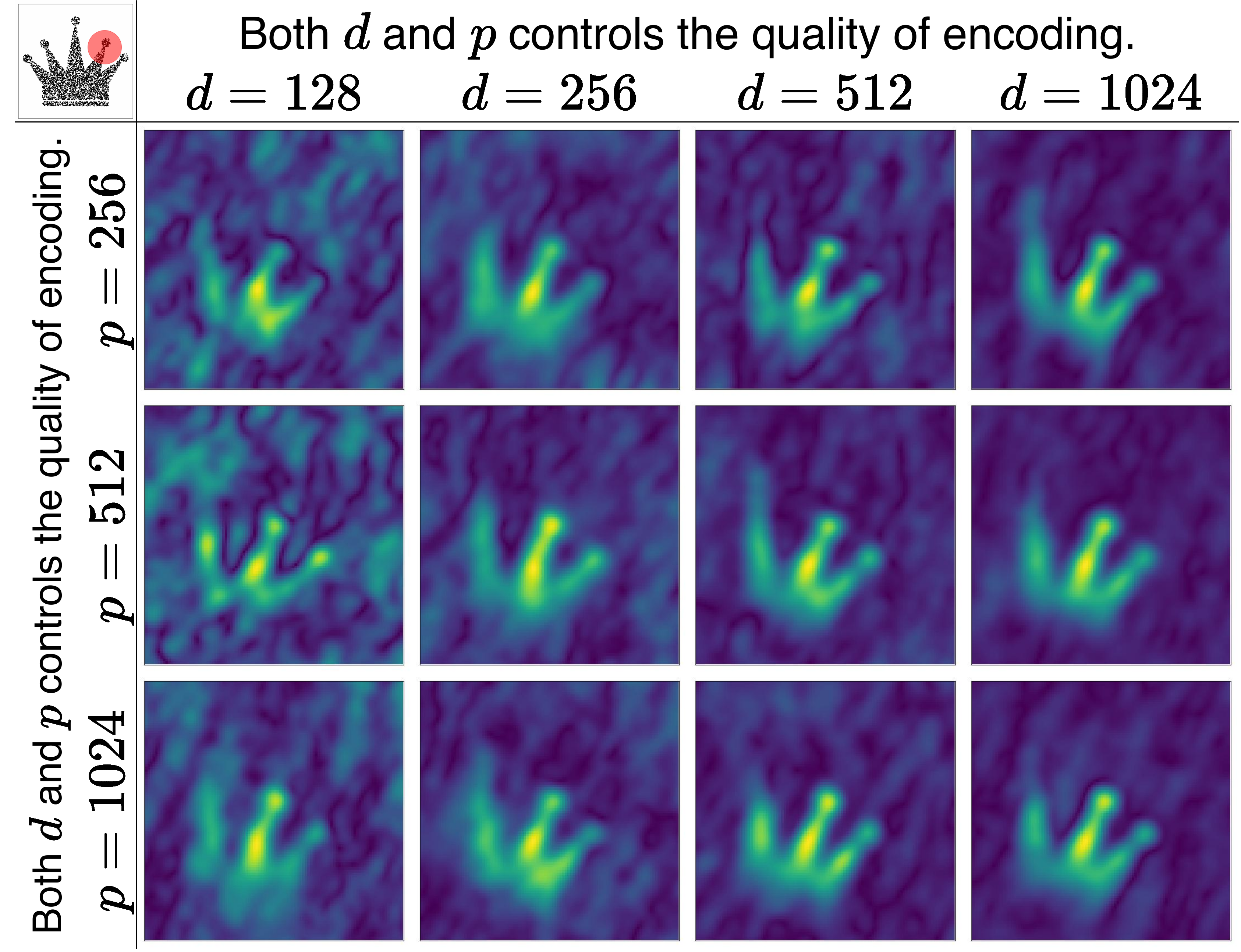}
    \vspace{-4pt}
    \caption{Effect of the parameters $d$ and $p$ in Theorem 
    \ref{thm:DLGE}.}
    \label{fig:dp}
    \vspace{-8pt}
\end{figure}

\newpage
\section{Guidance for Selecting the Parameter $\beta$}
\label{sec:radius_beta}
% Please add the following required packages to your document preamble:
% \usepackage{booktabs}
The statistics is obtained by $r=\min_{r}\{r: e^{-r^2/2} < 0.1\}$. The beta and the radius have a relation of $\beta_1r_1=\beta_2r_2$.
\begin{table}[h]
\centering
\caption{Relation between the parameter $\beta$ and the neighborhood radius.}
\label{tab:beta}
\resizebox{0.8\textwidth}{!}{%
\begin{tabular}{@{}lllllllllll@{}}
\toprule
beta   & 1     & 2     & 3     & 4     & 5     & 6     & 7     & 8     & 9     & 10    \\
radius & 1.800 & 0.900 & 0.600 & 0.450 & 0.360 & 0.300 & 0.257 & 0.225 & 0.200 & 0.180 \\ \midrule
beta   & 11    & 12    & 13    & 14    & 15    & 16    & 17    & 18    & 19    & 20    \\
radius & 0.163 & 0.150 & 0.138 & 0.129 & 0.120 & 0.113 & 0.106 & 0.100 & 0.095 & 0.090 \\ \midrule
beta   & 21    & 22    & 23    & 24    & 25    & 26    & 27    & 28    & 29    & 30    \\
radius & 0.086 & 0.082 & 0.078 & 0.075 & 0.072 & 0.069 & 0.067 & 0.065 & 0.062 & 0.060 \\ \bottomrule
\end{tabular}%
}
\end{table}

\end{document}